\documentclass[12pt]{article}
\usepackage[utf8]{inputenc}
\usepackage{amssymb}

\usepackage{hyperref}

\usepackage{todonotes}
\usepackage{fancyvrb}
\usepackage{amsmath}
\usepackage{amsthm}
\usepackage{stmaryrd}
\usepackage{pdfpages}
\usepackage[left = 3cm, right = 3cm]{geometry}
\usepackage{titling}
\usepackage{lipsum}
\usepackage{array}

\usepackage[toc,page]{appendix}

\newtheorem{theorem}{Theorem}
\newtheorem{lemma}[theorem]{Lemma}
\newtheorem{definition}[theorem]{Definition}
\newtheorem{proposition}[theorem]{Proposition}

\newtheorem*{remark}{Remark}

\usepackage{bm}

\usepackage{booktabs}       
\usepackage{multirow}

\usepackage{enumitem,booktabs}
\usepackage[referable]{threeparttablex}
\renewlist{tablenotes}{enumerate}{1}
\makeatletter
\setlist[tablenotes]{label=\tnote{\alph*},ref=\alph*,itemsep=\z@,topsep=\z@skip,partopsep=\z@skip,parsep=\z@,itemindent=\z@,labelindent=\tabcolsep,labelsep=.2em,leftmargin=*,align=left,before={\footnotesize}}
\makeatother

\newcommand{\R}[0]{\mathbb{R}}

\newcommand{\Z}[0]{\mathbb{Z}}

\newcommand*\samethanks[1][\value{footnote}]{\footnotemark[#1]}

\title{\texttt{Persformer}: A Transformer Architecture for Topological Machine Learning}
\author{
    Raphael Reinauer\thanks{Ecole Polytechnique Fédérale de Lausanne (EPFL) ,
        Laboratory for topology and neuroscience,
        CH-1015 Lausanne, Switzerland},
    Matteo Caorsi\thanks{L2F SA, Rue du centre 9, Saint-Sulpice, Switzerland} \thanks{The last two authors contributed equally to this work.} ,
    Nicolas Berkouk\samethanks[1] \samethanks[3]
}
\date{\today}

\begin{document}

\maketitle

\begin{abstract}
One of the main challenges of Topological Data Analysis (TDA) is to extract features from persistent diagrams directly usable by machine learning algorithms. Indeed, persistence diagrams are intrinsically (multi-)sets of points in $\R^2$ and cannot be seen in a straightforward manner as vectors. In this article, we introduce \texttt{Persformer}, the first Transformer neural network architecture that accepts persistence diagrams as input. The \texttt{Persformer} architecture significantly outperforms previous topological neural network architectures on classical synthetic and graph benchmark datasets. Moreover, it satisfies a universal approximation theorem. This allows us to introduce the first interpretability method for topological machine learning, which we explore in two examples.
\end{abstract}

\section{Introduction}

Topological Data Analysis (TDA) is a rapidly growing field of data science that incorporates methods to estimate the topology of a dataset within machine learning pipelines. The most common descriptors in the field are the so-called \emph{persistence diagrams}, which are further used in TDA pipelines for classification and regression tasks; the applications span a wide variety of scientific areas such as material science \cite{NHH15, Nanoporous}, neuroscience \cite{CliqueNeuron}, cancer biology \cite{aukerman_et_al:LIPIcs:2020:12169}, comprehension of deep learning architectures \cite{Nait20,lacombe2021topological} and the analysis of COVID19 propagation \cite{dlotko2020visualising}.  

Persistence diagrams are subsets of $\R^2$ whose points correspond to topological features in the dataset (such as connected components, loops, voids, ...), with coordinates encoding a notion of size of the feature. They are usually compared using either the bottleneck distance or a Wasserstein-type distance. We refer the reader to one of the many introductory TDA textbooks \cite{CTDA, Oudo15,EAT} for the details of these definitions. 

The main challenges for incorporating persistence diagrams into a machine learning pipeline are twofold. Firstly, the data structure underlying persistence diagrams is intrinsically a set. Hence the learned set representations should be invariant to the order in which points are presented. Secondly, whatever the type of distance considered (bottleneck or Wasserstein), the space of persistence diagrams cannot be isometrically embedded into a Hilbert space \cite[Theorem 4.3]{hilbertpersistence}. This is a major challenge because most machine learning algorithms are designed to operate on vectors in Hilbert spaces.

To overcome these issues, the TDA community has developed several \emph{vectorization} methods in order to associate to a set of persistence diagrams a set of vectors in a Hilbert space. These methods are primarely of two types. One either defines \emph{a priori} the vectorization map, as for persistence landscapes \cite{Bubenik2015StatisticalTD}, or one learns it through a trainable architecture, such as a neural networks \cite{Carrire2020PersLayAN,Hofer2019}. 

One of the major changes of paradigm in neural network architectures over the last five years concerns the introduction of the \emph{transformer} architecture \cite{Vaswani2017}, incorporating a \emph{self-attention} mechanism. In short, transformer models process the datum as a whole, exploit long-distance relationships between the elements of a datum and avoid recursion altogether. This is why transformers architectures achieve state-of-the-art performance on a variety of tasks. In Natural Language Processing (NLP), large pre-trained language models like \texttt{BERT} \cite{Devlin2018} and \texttt{GPT-3} \cite{GPT3} achieve state-of-the-art result on various NLP benchmarks. For computer vision tasks, the Vision Transformer \cite{zhai2021scaling} models achieve state-of-the-art results on several benchmarks including the ImageNet classification benchmark.

In this work, we introduce \texttt{Persformer}, a transformer neural network architecture designed for analyzing persistence diagram datasets, making available the power and versatility of transformer architecture for topological machine learning. We compare our model with already existing neural network architectures handling persistence diagrams: \texttt{PersLay} \cite{Carrire2020PersLayAN} and \texttt{PLLAy} \cite{PLLay}, exceeding the test accuracy of previous state-of-the-art models by 3.5\% and 4.0\%, respectively, on benchmark datasets.

The fact that our architecture does not make use of any handcrafted vectorization of persistence diagrams, as is the case for already existing methods \cite{Carrire2020PersLayAN, PLLay}, allows us to adapt a well-known interpretability method for neural networks to \texttt{Persformer}. We define Saliency Maps for \texttt{Persformer}, whose value on a given point of a persistence diagram quantifies the importance of this point for the classification task. In particular, we recover with Saliency Maps the observation made in \cite{Bubenik2020}, that the ``small bars'' of persistence diagrams detect curvature. These results lead us to conclude that it is too restrictive for topological machine learning tasks to  assume that ``small bars'' are mere representations of noise. Therefore, since ``small bars'' are unstable (i.e., they change easily from one realization of a datum to another), we conclude that it is too restrictive to impose stability of the topological features extracted by our \texttt{Persformer}.

\section{Related Works}\label{s:related_work}

As described in the introduction, one of the main challenges of topological data analysis is that the space of persistence diagrams, equipped with either the bottleneck or the Wasserstein distance, cannot be isometrically embedded into a Hilbert space \cite[Theorem 4.3]{hilbertpersistence}. Since most machine learning methods assume that the input dataset is a subset of a Hilbert space, they cannot be directly applied to datasets of persistence diagrams. To overcome this issue, considerable effort has been made in the TDA community to define \emph{vectorizations} of the space of persistence diagrams, that is, to define a Hilbert space $\mathcal{H}$ together with a continuous map $\phi : \mathcal{D} \to \mathcal{H}$, with $\mathcal{D}$ the space of persistence diagrams endowed with either the bottleneck or the Wasserstein distance. These methods are primarily of two types.

\paragraph{Prescribed vectorization methods} This corresponds to defining a Hilbert space $\mathcal{H}$, and a continuous map $\phi : \mathcal{D} \to \mathcal{H} $ that are independent of the machine learning task one tries to solve. Among others, there are the \emph{persistence scale-space kernel} \cite{reininghaus2014stable}, \emph{persistence landscapes} \cite{Bubenik2015StatisticalTD}, the \emph{weighted Gaussian kernel} \cite{hiraokagaussian}, or the \emph{sliced Wasserstein kernel} \cite{carrieresliced}.

\paragraph{Learnable vectorization methods} Another approach to vectorization methods of persistence diagrams is to learn the ``best" one for a fixed  data analysis task among a family of vectorization maps. More precisely, assume we are considering a fixed learning task (such as supervised classification) on a Hilbert space $\mathcal{H}$, and a family of vectorizations $\phi_\theta : \mathcal{D} \to \mathcal{H}$. Then this approach consists in learning the best value of the parameter $\theta$, according to an optimization criterion provided by the learning task (typically a loss function associated to the classification process in $\mathcal{H}$). 

The first article introducing learnable vectorization methods of persistence diagrams is \cite{NIPS2017_883e881b}, where the $\phi_\theta$ are given by the sum of two-dimensional Gaussian functions (with mean and standard deviation defined by $\theta$) evaluated on the points of persistence diagrams. In \cite{Hofer2019}, the authors introduce the first neural network architecture able to accept persistence diagrams as input. Furthermore, they elaborate on the observation (initially made in another context \cite[Theorem 2]{Zaheer2017}) that any  real-valued Hausdorff-continuous function on the set of persistence diagrams contained in a fixed compact subset of $\R^2$ with exactly $n$ points can be approximated arbitrarily well by

$$L'(\{x_1,...,x_n\}) := \rho \left (\sum_{i=1}^n \phi(x_i) \right ), $$



\noindent for certain functions $\phi : \R^2 \to \R^p $ and $\rho : \R^p \to \R$. They introduce specific classes of functions for $\rho$ and $\phi$, which have then been extended in \cite{Carrire2020PersLayAN} by the \texttt{PersLay} architecture.

\section{Background}

\subsection{Persistence diagrams}

Persistence diagrams are the most commonly used descriptors developed by the TDA community. They come in two main flavors: ordinary and extended persistence diagrams. We refer to the textbook \cite{CTDA} for an extended exposition of the mathematical background of this section.

Ordinary persistence diagrams track the evolution of topological features of dimension $i$ (connected component for $i=0$, holes for $i=1$, cavities $i=2$, ...) in nested sequences of topological spaces (or simplicial complexes) $(X_t)_{t\in \R}$, where $X_a \subset X_b$ whenever $a \leq b$. The variable $t \in \R$ is called the   \emph{filtration value} and intuitively corresponds to the time of the evolution of the topological features. The appearance or disappearance of different topological features depending on $t$ is the birth of new topological features or death. If a topological feature of dimension $i$ is born at time $b \in \R$ in the filtration and dies at time $d \in \R \cup \{+\infty\}$, it will give rise to a point with coordinate $(b,d)$ in the $i$-th persistence diagram of this filtration. Therefore, ordinary persistence diagrams are multi-sets (sets where elements can have multiplicity) of points in the subset $\{(b,d) \in \R \times \R\cup \{+\infty\} \mid b < d\}$ of $\R \times \R\cup \{+\infty\}$ \cite[Section 3.2.1]{CTDA}. 

Extended persistence diagrams generalize ordinary persistence, and subsume the size and type of topological features of the fibers of a continuous map $f : X \to \R$. In practice, extended persistence diagrams are defined for real-valued functions on the $0$-simplices of a simplicial complex. For each dimension $i$, the $i$-th dimensional topological features of the fiber of $f$ are encoded with a birth $b\in \R$ and a death $d \in \R$, and one of the following four possible types: Ordinary, Relative, Extended+ or Extended-- \cite{extending}. Extended persistence is a strict generalization of ordinary persistence since the latter is contained in the former. Furthermore, it has the computational advantage of containing only points with finite coordinates.

For simplicity, we will treat persistence diagrams as sets and not multi-sets, that is, we will assume all points in a persistence diagram to be disjoint. Let $X \subset \{(b,d) \in \R^2 \mid b < d \}$.

\begin{definition}
The set of persistence diagrams on $X$ is defined by \[PD(X) := \{D \subset X \mid \forall K \subset X ~\textnormal{compact}, D \cap K ~\textnormal{is finite}\}.\]

\noindent Given $n \in \Z_{>0}$, we also define the set of persistence diagrams with $n$ points by:

\[PD_n(X) := \{D \in PD(X) \mid D~\textnormal{has $n$ elements}\}.\]
\end{definition}

\subsection{Metrics on persistence diagrams}

It is possible to compare persistence diagrams using various distances, all defined as the infimum cost of a partial matching problem between points of two persistence diagrams. There are mainly two classes of matching rules. In the first case, one has to match all points of the first diagram in a one-to-one correspondence with the points in the second one. In the second case, one only looks for partial bijection between points of the persistence diagrams, the unmatched points being matched to their projection on the diagonal  $\Delta = \{(x,x) \mid x \in \R\}$. Both persistence diagrams and  distances can be efficiently computed by software such as \emph{Giotto-tda} \cite{tauzin2021giottotda}, \emph{Gudhi} \cite{gudhi} or \emph{Dionysus} \cite{Dionysus}.

For $p \in \R_{\geq 1}$, and $x = (x_1,...,x_n) \in \R^n$, we denote by $\|x\|_p$ the $p$-norm of $x$ defined by $\|x\|_p = (\sum_i x_i^p)^{\frac{1}{p}}$. For $p = \infty$, we set $\|x\|_\infty = \max_i |x_i|$. Given $D, D' \in PD_n(X)$ and $\sigma : D \longrightarrow D'$ a bijection, we define $c(\sigma) \in \R^n$ by choosing an ordering $D = \{z_1, ... , z_n\}$, and setting $c(\sigma) = (\|z_1 - \sigma(z_1)\|_\infty, ... , \|z_n - \sigma(z_n)\|_\infty)$. Note that the use that we will make of $c(\sigma)$ is independent of the ordering we have picked on $D$.

\begin{definition}
Let $n\in \Z_{>0}$, $p \in \R_{\geq 1} \cup \{\infty\}$ and $D,D' \in PD_n(X)$. The $p$-Wasserstein distance between $D$ and $D'$ is defined by: 

\[W^p(D,D') := \min_{\sigma : D \stackrel{\sim}{\to} D'} \|c(\sigma)\|_p,\]

\noindent where $\sigma$ ranges over all bijections between $D$ and $D'$.
\end{definition}

$W^\infty$ is commonly called the Haussdorff distance.

\begin{proposition}\label{p:pnorms}
Let $n\in \Z_{>0}$ and $p,q \in \R_{\geq 1} \cup \{\infty\}$, there exists two strictly non-negative constants $m(p,q)$ and $M(p,q)$ such that for all $D,D' \in PD_n(X)$, one has:

\[m(p,q) \cdot W^q(D,D') \leq W^p(D,D') \leq M(p,q) \cdot W^q(D,D').\]
\end{proposition}

\begin{proof}
This is a direct consequence of the equivalence of all norms on a finite dimensional real vector space. 
\end{proof}

Therefore, the topology induced by $W^p$ on $PD_n(X)$ is independent of $p$.

Given $x \in \R^2$, we denote by $\pi(x)$ the orthogonal projection of $x$ onto the diagonal $\Delta$. For $D,D' \in PD_n(X)$, a partial matching between $D$ and $D'$ is the data of two possibly empty subsets $I \subset D$ and $I' \subset D'$, together with a bijection $\sigma : I \longrightarrow I'$. We will use the notation $(\sigma, I, I') : D \longrightarrow D'$. Given $(\sigma, I, I') : D \longrightarrow D'$, where $I$ and $I'$ have $\ell$ elements, we choose an ordering $D = \{z_1, ..., z_n\}$ and similarly $D' = \{z'_1,...,z'_n\}$, where $z_i \in I$ and $z'_i \in I'$ for all $i \leq \ell$. We set $c(\sigma, I, I') \in \R^{2n - \ell}$ defined by:

\[c(\sigma,I,I')_i = \begin{cases} \|z_i - \sigma(z_i)\|_\infty ~\textnormal{if}~i\leq \ell \\

\|z_i - \pi(z_i)\|_\infty ~\textnormal{if}~\ell + 1 \leq i\leq n \\
\|z'_{i - (n - \ell)} - \pi(z'_{i - (n - \ell)})\|_\infty ~\textnormal{if}~ n + 1 \leq i\leq 2n - \ell

\end{cases}.
\]

\begin{definition}
Let $n\in \Z_{>0}$, $p \in \R_{\geq 1} \cup \{\infty\}$ and $D,D' \in PD_n(X)$. The diagonal-$p$-Wasserstein distance between $D$ and $D'$ is defined by:

\[W^p_d(D,D') := \min_{(\sigma,I,I') : D \to D'} \|c(\sigma, I, I')\|_p,\]

\noindent where $(\sigma, I, I')$ ranges over all partial matchings between $D$ and $D'$.
\end{definition}

$W_d^\infty$ is usually called the bottleneck distance \cite[Definition 3.9]{CTDA}.

\begin{proposition}\label{p:diadnondiag}
Let $n\in \Z_{>0}$ and $p \in \R_{\geq 1} \cup \{\infty\}$. The topologies induced on $PD_n(X)$ by $W^p $ and $W^p_d$ are the same.
\end{proposition}

\begin{proof}
Because matchings are partial matchings, for all $D,D' \in PD_n(X)$, one has $W_d^p(D,D') \leq W^p(D,D')$. Therefore, any $W^p$-open subset of $PD_n(X)$ is $W^p_d$-open. 

To prove the converse, it is sufficient to prove that for all $D \in PD_n(X)$, there exists $\varepsilon_D > 0$ such that for all $0 < \varepsilon \leq \varepsilon_D$:
\[\left \{D' \in PD_n(X) \mid W^p_d(D,D') < \varepsilon \right \} \subseteq \left \{D' \in PD_n(X) \mid W^p(D,D') < \varepsilon \right \}.  \]

Let $D \in PD_n(X)$. We define $\varepsilon_D := \min_{z \in D} \| z - \pi(z)\|_\infty > 0$, and let $0 < \varepsilon \leq \varepsilon_D$. Let $D' \in PD_n(X)$ be such that $W_d^p(D,D') < \varepsilon \leq \varepsilon_D$. Let $(\sigma, I, I') : D \longrightarrow D'$ be such that $\|c(\sigma, I, I')\|_p =W_d^p(D,D')$. Then $I$ has to be equal to $D$, because otherwise, we have to match a point $z\in D$ to it's projection onto the diagonal, and hence one would have $\varepsilon \leq \varepsilon_D \leq \|z - \pi(z)\|_\infty \leq \|c(\sigma, I, I')\|_p = W^p_d(D,D')$. Therefore, $\sigma$ is a bijection defined on $D$, satisfying $\|c(\sigma)\|_p < \varepsilon$ and hence, $W^p(D,D') < \varepsilon$. This proves the desired inclusion.
\end{proof}

\subsection{Approximating functions on sets}\label{trans_univ_approx}

We recall a useful approximation result for functions on sets.


\begin{theorem}[Theorem 9  \cite{Zaheer2017}]\label{thm:zaheer_approx}
Let $X$ be a compact subset of $\R^d$, and $M$ be a positive integer. Let $2^X_M \subset 2^X$ denote the set of subsets of $X$ with exactly $M$ elements, equipped with the Hausdorff metric.

For any Hausdorff continuous function $L : 2^X_M \to \R$ and $\varepsilon > 0$, there exist a natural number $p > 0$ and two continuous functions $\phi : X \to \R^p$ and $\rho : \R^p \to \R$ such that: $$ \sup_{S \in 2^X_M} \left|\rho \left(\sum_{x\in S} \phi(x)\right) -  L(S) \right| \leq \varepsilon. $$
\end{theorem}

\begin{remark}
The previous theorem is also true for multi-sets.
\end{remark}

Using the universal approximation theorem for neural networks \cite{Haykin2010NeuralNA} the statement of Theorem \ref{thm:zaheer_approx} can be extended to the statement that every Hausdorff continuous function $L$ on uniformly bounded finite subsets of $\R^2$ of cardinality $M$ can be arbitrarily well approximated by

\begin{align}\label{eq:sum_decompositon}
    L'(\{x_1,...,x_n\}) := \rho \left (\sum_{i=1}^n \phi(x_i) \right )
\end{align}
where $\phi$ and $\rho$ are neural networks with a finite number of hidden layers containing a finite number of neurons with a non-constant, bounded, and non-decreasing continuous activation function like ReLU. A neural network like (\ref{eq:sum_decompositon}) is called a Deep Set \cite{Zaheer2017}.

The encoder part of the transformer architecture introduced in \cite{Vaswani2017} without positional encoding and with multi-head attention pooling (see Section \ref{sec:pers_arc}) composed with a fully-connected neural network is at least as expressive as a Deep Sets model, see \cite{Lee2018}. Hence, this architecture satisfies the \emph{universal approximation theorem of set transformers}.
Moreover, as noted in \cite{Lee2018}, the self-attention mechanism enables explicit interactions between instances of a set and also higher-order interactions by stacking multiple layers. The authors of the paper further show state-of-the-art performance of this architecture on various set-based datasets.

\section{The \texttt{Persformer} architecture}
\label{sec:persformer}

This section is devoted to introducing the \texttt{Persformer} architecture in detail. We refer to \cite{Vaswani2017} for a detailed introduction of the self-attention mechanism.

\subsection{The building blocks of the \texttt{Persformer} architecture}\label{sec:pers_arc}

\begin{figure}[!ht]\label{fig:pers_arc}
    \centering
    \includegraphics[width=9cm]{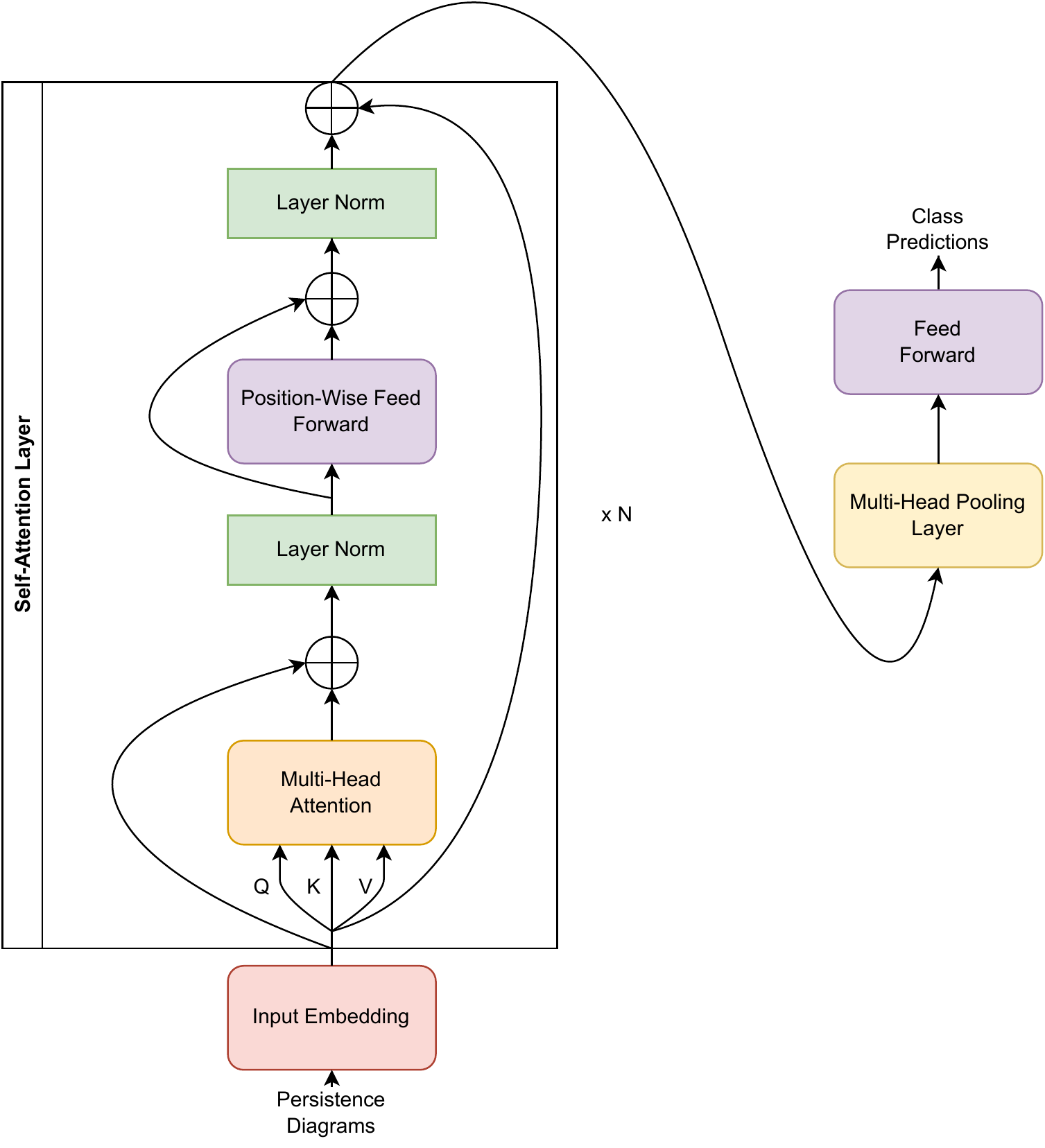}
    \caption{Diagram of the \texttt{Persformer} architecture.}
\end{figure}

The \texttt{Persformer} architecture is shown in Figure \ref{fig:pers_arc}. 
It consists of an embedding layer, which is a trainable position-wise fully connected layer\footnote{The position-wise feed-forward network consists of a fully connected neural network with two layers that is to all vectors in the sequence.} $\R^2\to \R^d$, followed by stacked self-attention layers consisting of a multi-head self-attention block and a fully connected feed-forward layer. Finally, the multi-head attention pooling layer provides a vector representation of the persistence diagram, and a fully connected neural network computes the final class prediction. The individual building blocks of the architecture are explained in more detail in the following sections.

\paragraph{Self-attention block} The self-attention mechanism aims to model pairwise interactions of elements in a sequence.
For this purpose, three families of vectors $\bm{Q}, \bm{K}, \bm{V}\in \R^{N \times d'}$ are calculated from a sequence of length $N$ of $d$-dimensional vectors $\bm{X}\in \R^{N \times d}$. These families are called \emph{query}, \emph{key}, and \emph{value} vectors. They are linear transformations of the input sequence $\bm{X}$ by trainable matrices $\bm{W}_Q,\,\bm{W}_K, \, \bm{W}_V \in \R^{d\times d'}$. For each query vector $\bm{Q}_i$, a similarity score is computed with all key vectors by calculating the scalar product up to a factor of $1/\sqrt{d'}$, and then all the similarity scores are normalized by using the softmax function to get the so-called attention score
\[
\mathrm{AttentionScore}(\bm{Q}_i, \bm{K}) = \mathrm{softmax}\left(\frac{\bm{Q}_i \bm{K}^T}{\sqrt{d'}}\right) \in \R^N.
\]
The output of the self-attention block $\mathrm{Attention}(\bm Q, \bm K, \bm V)$ is a sequence of length $N$ consisting of $d'$-dimensional vectors, where the $i$-th vector is given by a convex combination of the value vectors
\[
\mathrm{Attention}(\bm Q, \bm K, \bm V)_i = \sum_{j=1}^N \mathrm{AttentionScore}(\bm{Q}_i, \bm{K})_j \bm{V}_j.
\]

\paragraph{Multi-head Attention}
A multi-head attention block combines several self-attention blocks in parallel and enables joint attention to different parts of the input sequence \cite{Vaswani2017}.

To this end, the query, key, and value vectors are split into a family of $H$ sequences of vectors $\bm Q^{(h)}, \bm K^{(h)}, \bm V^{(h)}$ of size $\R^{d'/H}$, where $H$ is the number of attention-heads and $h=1, \ldots, H$. Here we assume that $H$ divides $d'$. Furthermore, the scalar-products of query and key vectors are multiplied by a factor $\sqrt{H/d'}$. For each head the attention vectors
\[
\mathrm{head_h} = \mathrm{Attention}(\bm Q^{(h)}, \bm K^{(h)}, \bm V^{(h)}) \in \R^{N\times (d'/H)}
\]
for $h=1, \ldots, H$ are computed and then combined to an output sequence
\[
\mathrm{MultiHead}(\bm Q,\bm K,\bm V) = \mathrm{Concat}(\mathrm{head_1}, \ldots, \mathrm{head_H}) \bm W^O \in \R^{N \times d}
\]
where $\bm W^O \in \R^{d'\times d}$ is a trainable linear transformation.

\paragraph{Position-wise Feed-Forward Network}
The position-wise feed-forward network is a fully connected neural network with two layers that is applied to each vector in the sequence.

\paragraph{Multi-head Attention Pooling}
The multi-head attention pooling layer is a variation of the multi-head attention layer with a single trainable query vector $\bm Q\in \R^{1\times d'}$ and key and value vectors as linear transformations of the input sequence \cite{Lee2018}. The output
\[
\mathrm{MultiHead}(\bm Q, \bm K, \bm V) \in \R^d
\]
is a single vector that does not depend on the order of the input sequence.

\paragraph{Residual connections}
Since the points in a persistence diagram may be very close to each other, it may be difficult for the encoder to separate them. This difficulty leads the gradients of the attention blocks to be very small at the beginning, which complicates the training of this architecture, especially with the numerous self-attention layers \cite{goodfellow2016deep}.

To overcome this problem, we added residual connections between the self-attention layers. To our knowledge, this architecture choice has not been made so far in the literature. Empirically, we could see that this solved the vanishing gradient problem and allowed us to stack significantly more self-attention layers while substantially speeding up the training.

\subsection{Permutation-invariance property}
 On a high level, the \texttt{Persformer} architecture consists of an \emph{encoder} ($\phi$), a \emph{pooling layer} ($p$), and a \emph{decoder} ($\rho$) that computes the final class prediction. As input, we take the points in the persistence diagram together with their one-hot encoded homology dimensions.
 
 The encoder consists of stacked attention layers. Each attention layer maps a sequence of vectors to a sequence of vectors and is permutation-equivariant, i.e., for every $\{x_1, \ldots, x_n\}$ sequence of vectors and permutation $\sigma$ of the set $\{1, \ldots, n\}$ we have
\[
\phi(\{x_{\sigma(1)}, \ldots, x_{\sigma(n)}\}) = \sigma(\phi(\{x_{1}, \ldots, x_{n}\})),
\]
where $\sigma$ permutes the order of the vectors in the sequence $\phi(\{x_{1}, \ldots, x_{n}\})$.
The attention pooling maps a sequence of vectors to a single vector in a permutation-invariant way, i.e., the output vector does not depend on the order of the sequence.
Combining the permutation-equivariance property of the encoder and the permutation-invariance of the attention layer, we get a permutation-invariant map $\rho \circ p \circ \phi$ that maps a sequence of vectors to a single vector.

\subsection{Universal approximation theorem of \texttt{Persformers}}

We state the following technical lemma, whose proof is a consequence of Heine's theorem on uniform continuity of continuous functions with compact supports.

\begin{lemma}\label{lem:uniformconv}
Let $X \subset \R^a$ be a compact subset. For $k\geq 0$, let $\phi_k,\phi : X \longrightarrow \R^b$ be continuous maps such that $K := \overline{ \cup_k \phi_k(X)\cup \phi(X)}$ is compact, and let $\rho_k,\rho : K \longrightarrow \R^c$ be continuous maps. If

\[ \sup_{x \in X} \|\phi_k(x) - \phi(x)\| \underset{k \longrightarrow +\infty}{\longrightarrow} 0 \quad \textnormal{and} \quad \sup_{x \in K} \|\rho_k(x) - \rho(x)\| \underset{k \longrightarrow +\infty}{\longrightarrow} 0,  \]

\noindent then one has for every $n\geq 1$:

\[ \sup_{(x_1,...,x_n) \in X^n} \left \|\rho_k \left (\sum_{i=1}^n \phi_k(x_i) \right ) - \rho \left (\sum_{i=1}^n \phi(x_i) \right ) \right \| \underset{k \longrightarrow +\infty}{\longrightarrow} 0.\]

\end{lemma}

\begin{theorem}[Universal approximation theorem of \texttt{Persformer}]
Let $X \subset \R^2$ be compact and let $PD_n(X)$ be the space of persistence diagrams contained in $X$, consisting of at most $n$ points and endowed with the topology induced by any of the distances $W^p$ or $W_d^p$ ($p\in \R_{\geq 1}$).

Then, every continuous function $f: PD_n(X)\to \R$ can be uniformly approximated by a \texttt{Persformer} model with ReLU-activations and a fixed hidden dimension of encoder layer $2n+1$.
\end{theorem}

\begin{proof}
By propositions \ref{p:pnorms} and  \ref{p:diadnondiag}, the topologies induced by the distances $W^p$ and $W^p_d$ are all the same. Therefore, we can assume that  $f : PD_n(X) \to \R$ is a $W^\infty$ (\emph{i.e}. Haussdorff) continuous function.  Let $\varepsilon > 0$. Then by Theorem \ref{thm:zaheer_approx}, there exists continuous maps $\phi : \R^2 \longrightarrow \R^{2n + 1}$ and $\rho : \R^{2n + 1} \to \R $ such that for all $x_1,...,x_n \in X$, 

\[\left|\rho \left(\sum_{i = 1}^n \phi(x_i)\right) - f\left (\{x_1,...,x_n\} \right )\right| \leq \frac{\varepsilon}{2}. \]

\noindent Keeping notations of the previous section, we now define a sequence of  \texttt{Persformer} models $F_k : (\R^2)^n \to \R$ of the following form: \[F_k(x_1,...,x_n) = \rho_k \circ p_k \circ (\phi_k (x_1),...,\phi_k(x_n)). \]

It should be noted that when attention-scores are set to $0$, an attention layer is equivalent to a standard fully-connected feed-forward one. Therefore $\phi_k$, which is the concatenation of the input embedding layer and the self-attention blocks together with residual connections, is strictly more expressive than a feed-forward network with residual connections and $2n+1$ hidden neurons per layer. Therefore according to \cite{DBLP:journals/corr/abs-2007-06007}, we can choose $\phi_k$  such that: 

\[\sup_{x \in X} \|\phi_k(x) - \phi(x)\| \underset{k \longrightarrow +\infty}{\longrightarrow} 0. \]

Similarly, when setting the query vector to $0$, the attention pooling computes the mean of the vector $\phi_k(x)$, which we can consider to be simply the sum of the coordinate elements of the vector, after renormalization by the feed-forward layer $\rho_k$. Consequently, we obtain the following expression:

\[F_k(x_1,...,x_n) = \rho_k \left ( \sum_{i=1}^n \phi_k(x_i)\right ). \]

Finally, by the universal approximation theorem for fully-connected layers \cite{Haykin2010NeuralNA}, we can choose $\rho_k$  such that: 

\[\sup_{x \in X} \|\rho_k(x) - \rho(x)\| \underset{k \longrightarrow +\infty}{\longrightarrow} 0. \]

By lemma \ref{lem:uniformconv}, we can conclude that:

\[ \sup_{(x_1,...,x_n) \in X^n} \left \|F_k(x_1,...,x_n) - \rho \left (\sum_{i=1}^n \phi(x_i) \right ) \right \| \underset{k \longrightarrow +\infty}{\longrightarrow} 0.\]

In particular, there exists an integer $N$ such that for all $k \geq N$: 

\[ \sup_{(x_1,...,x_n) \in X^n} \left \|F_k(x_1,...,x_n) - \rho \left (\sum_{i=1}^n \phi(x_i) \right ) \right \| \leq \frac{\varepsilon}{2}.\]

Therefore, 

\[\sup_{(x_1,...,x_n)\in X^n} \|F_N(x_1,...,x_n) - f(\{x_1,...,x_n\})\| \leq \varepsilon. \]

\end{proof}

\subsection{Training of \texttt{Persformers}}

\paragraph{Training specification}
Unlike the optimization of other neural network architectures, a learning warm-up stage is crucial to achieving good results for transformer architectures \cite{Popel2018}. Empirically, we have found that this results in an increase of about 2 percentage points in test accuracy for the \texttt{Persformer} model.
We used the optimizer AdamW with a weight decay of 1e-4 and a maximum learning rate of 1e-3, a batch-size of 32, and a cosine with hard restarts learning-rate scheduler with 10 warm-up epochs, 3 cycles, and a total of 1,000 epochs. Using grid-search, we found that with a dropout of 0.2 in the decoder part of the \texttt{Persformer} and no dropout in the encoder part we obtained the best results.

\paragraph{Model specification}
We used a \texttt{Persformer} model with residual connections, 5 encoder layers with hidden-dimension $d=128$, 8 attention heads per layer, trainable layer normalization, multi-head attention pooling, and GELU-activation, as well as a decoder which is a fully connected neural network with layer sizes
$$[128,\,256,\,256,\,64,\,5].$$ 

\subsection{Benchmarks}\label{s:benchmark}

\subsubsection{The \texttt{ORBIT5k} dataset}
To demonstrate the performance of our models, we consider the \texttt{ORBIT5k} dataset which is a standard dataset to benchmark methods for vectorizing persistence diagrams \cite{Adams2017, PLLay, Carrire2020PersLayAN}. The dataset consists of subsets of size 1,000 of the unit cube $[0, 1]^2$ generated by a dynamical system that depends on an parameter $\rho>0$. To generate a point cloud, a random initial point $(x_0, y_0)\in [0, 1]^2$ is chosen randomly in $[0, 1]^2$ and then a sequence of points $(x_n, y_n)$ for $n=0, 1, \ldots, 999$ is generated recursively by:

\begin{align*}
x_{n+1} &= x_n + \rho y_n(1 - y_n) & \mathrm{mod}\ 1,\\
y_{n+1} &= y_n + \rho x_{n+1} (1-x_{n+1}) & \mathrm{mod}\ 1.
\end{align*}

For every parameter $\rho = 2.5, 3.5, 4.0, 4.1$ and $4.3$ we generated a dataset of 1,000 orbits obtaining 5,000 orbits at the end. The shape of the point cloud depends heavily on the parameter $\rho$. The orbits are transformed into persistent diagrams by computing the alpha complex filtration in dimensions 0 and 1 \cite{Carrire2020PersLayAN}. The classification problem is to recover the parameter $\rho\in \{2.5, 3.5, 4.0, 4.1, 4.3\}$ from the given persistence diagrams. 

In addition to the \texttt{ORBIT5k} dataset, we consider the  \texttt{ORBIT100k} dataset, containing 20,000 orbits per parameter instead of 1,000 orbits per parameter, for a total of 100,000 orbits.
Both datasets are split in a ratio of 70:30 into training and test sets.

Besides classical kernel methods, we compare the performance of our \texttt{Persformer} model with current state-of-the-art models -- the \texttt{PersLay} \cite{Carrire2020PersLayAN} model and the \texttt{PLLAy} model \cite{PLLay}. Since \texttt{PLLay} uses both the persistence diagrams and the raw point clouds as inputs, we trained a \texttt{Persformer}-like model using only the raw point clouds as input, achieving a test accuracy of 99.1\%, an increase of 4.1 percentage points compared to \texttt{PLLay}. 
We repeated all experiments for the best performing architecture five times and report the average performance and the standard deviation on a holdout test dataset. A detailed comparison with all the models is in Table \ref{tab:persformer-results}.

\subsubsection{\texttt{MUTAG} graph dataset}
We also use the \texttt{MUTAG} dataset to evaluate our model's performance on graph classification and compare it to the \texttt{Perslay} model. The \texttt{MUTAG} dataset consists of 188 graphs from chemical compounds labeled as mutagenic or non-mutagenic. Each graph contains between 17 and 28 nodes. The \texttt{MUTAG} dataset has been used in many studies and is a standard benchmark for graph classification. For a direct comparison to the state-of-the-art, the same setup as in \cite{Carrire2020PersLayAN}  is used in our experiment. 

As a filtration function, we used the \emph{Heat Kernel Signature} defined on node $v$ in graph $G$ as
\[
\mathrm{hks}_t(v) = \sum_{k=1}^n \exp(-t\lambda_k)\varphi_k(v)^2,
\]
with $L = I_n - D^{-\frac 12} A D^{\frac 12}$ being the \emph{normalized graph Laplacian} with eigenfunctions $\varphi_1, \ldots,\, \varphi_n$ and eigenvalues $0\le \lambda_1 \le \ldots \le \lambda_n \le 2$, where $I_n$ is the identity matrix and $A$ is the adjacency matrix of the graph. We use the diffusion parameter $t=10.0$ and extended persistent homology $\mathrm{Ord}_0,\, \mathrm{Rel}_1,\, \mathrm{Ext}_0^+,\, \mathrm{Ext}_1^-$ resulting in six-dimensional input vectors where the first two dimensions are the birth- and death-times and the last four are the one-hot-encoded homology types.

The model we used for the \texttt{MUTAG} dataset is small, with only two layers of size $d=32$ and four attention heads per layer. We use GELU-activation and trainable layer normalization, as well as a decoder which is a fully connected neural network with layer sizes
$$[64,\,32,\,2].$$ We found a better model performance when combining attention pooling with sum pooling instead of using attention pooling alone.  We computed the sum and attention pooling for the last layer of the encoder and concatenated the results.

We evaluate the model by performing 10-fold cross-validation and report the average accuracy on the hold-out validation set. We compare our model to the \texttt{PersLay} model that is only trained on the extended persistence diagrams since we are interested in the ability of our model to learn expressive features from the extended persistence diagrams. The results of our experiment are in Table \ref{tab:persformer-results-mutag}.
\begin{table}[t]
\centering
\caption{Average validation accuracy and standard deviation across 10-fold cross-validation. Both models are only trained on the extended persistence diagrams.}
\vspace{0.3cm}
\label{tab:persformer-results-mutag}
\begin{tabular}{@{}lcc@{}}
\toprule
Model & \texttt{MUTAG}  \\ \midrule
\texttt{PersLay} & $85.8\% (\pm 1.3)$ \\
\texttt{Persformer} & $\mathbf{89.9}\% (\pm 2.1)$ \\ \bottomrule
\end{tabular}
\end{table}

\begin{table}[htbp]
\begin{center}
\caption{The \texttt{Persformer} model achieves better test accuracy on the \texttt{ORBIT5k} and \texttt{ORBIT100k} datasets than previous state-of-the-art models for point cloud inputs (pc), persistence diagram inputs (pd) and point cloud+persistence diagram inputs (pc+pd).}
\label{tab:persformer-results}
\vspace{+2mm}
\begin{threeparttable}
\begin{tabular}{p{55mm}cccccc}
\toprule
\multirow{2}{*}{\vspace{-2mm}Model} & \multicolumn{3}{c}{\texttt{ORBIT5k}} & & & \multicolumn{1}{c}{\texttt{ORBIT100k}} \\
\cmidrule{2-4} \cmidrule{6-7} 
& pc & pd & pd+pc & &  & pd\\ 
\hline
PointNet \cite{Qi} & 70.8\%\tnotex{tnote:pllay} & & & & &\\
CNN \cite{PLLay} & 91.5\%\tnotex{tnote:pllay} & & & & &\\
\texttt{Pllay} + CNN \cite{PLLay} & & & 94.5\%\tnotex{tnote:pllay} & & &\\
\texttt{Pllay} + CNN \cite{PLLay} & & & 95.0\%\tnotex{tnote:pllay} & & &\\
\hline
Persistence Scale Space Kernel \cite{reininghaus2014stable} &  & 72.38\%\tnotex{tnote:perslay} & & & &\\
Persistence Weighted Gaussian Kernel \cite{Kusano2018} &  & 76.63\%\tnotex{tnote:perslay} & & & &\\
Sliced Wasserstein Kernel \cite{Carriere2017} &  & 83.6\%\tnotex{tnote:perslay} & & & &\\
Persistence Fisher Kernel \cite{Le2018} &  & 85.9\%\tnotex{tnote:perslay} & & & &\\
\texttt{PersLay} \cite{Carrire2020PersLayAN} &  & 87.7\%\tnotex{tnote:perslay} & & & & 89.2\tnotex{tnote:perslay}\\
\specialrule{1pt}{-1pt}{0pt}
\texttt{Persformer} w/o layer-norm (ours) & 97.8\%  & 90.4\%& 98.2\% & & &91.0\%\\
\texttt{Persformer} w/ layer-norm (ours) & 96.4\% & --\tnotex{tnote:unable}  & 96.1\% &  & &--\tnotex{tnote:unable}\\
\begin{tabular}{@{}c@{}}\texttt{Persformer} w/ layer-norm \\+ residual connections (ours)\end{tabular} & \begin{tabular}{@{}c@{}}\textbf{99.1}\%\\($\pm$ 0.3)\end{tabular} & \begin{tabular}{@{}c@{}}\textbf{91.2}\%\\($\pm$ 0.8)\end{tabular}& \begin{tabular}{@{}c@{}}\textbf{99.1}\%\\($\pm$ 0.3)\end{tabular}& & &  \begin{tabular}{@{}c@{}}\textbf{92.0}\%\\($\pm$ 0.4)\end{tabular}\\
\bottomrule
\end{tabular}
\begin{tablenotes}
  \item\label{tnote:pllay}Score as reported in \cite{PLLay}.
  \item\label{tnote:perslay}Score as reported in \cite{Carrire2020PersLayAN}.
  \item\label{tnote:unable}The model could not be trained and had accuracy similar to random guessing.
\end{tablenotes}
\end{threeparttable}
\end{center}
\end{table}

\section{Interpretability method for \texttt{Persformer}} \label{sec:interpretability_persformer}

Compared to previous methods \cite{Carrire2020PersLayAN, PLLay}, our model does not make use of handcrafted vectorization of persistence diagrams in the first layer of the neural network, which allows \texttt{Persformer} to satisfy the universal approximation property in the sense of Section \ref{trans_univ_approx}. This allows us to identify those points that are important for the classification, by the mean of Saliency Maps, an already well-known interpretability method. We will show with concrete datasets (\texttt{Orbit5k},\, \texttt{Orbit100k} and a curvature dataset \cite{Bubenik2020}) that, contrary to the common view, ``small bars'', which are unstable features of persistence diagrams with respect to bottleneck or Wasserstein distances,  are also essential predictors for classifications and regression problems, as already identified in \cite{Bubenik2020}.

\paragraph{Saliency Maps}

The \texttt{Persformer} model for a classification problem is an almost everywhere differentiable function $F: \mathcal D \to \R^m$, where $m$ is the number of classes and $\mathcal D$ is the space of persistence diagrams. It maps a persistence diagram to the logarithm of the class probability. Let $d$ be the maximum homology dimension to be considered and let $x = (x_k)_{k\in \{ 1,\ldots, n \}}\in (\R^{2+d})^n$ be a persistence diagram and $i(x) = \mathrm{argmax}_j F(x)_j$. The first two coordinates of $x_k \in \R^{2+d}$ are the birth and death coordinates and the last $d$ coordinates are the one-hot encoded homology dimensions. The \emph{saliency map} of $F$ on $x$ is defined as 
\[ \mathcal{S}_F(x) :=
\left (\left\|\frac{\partial F_{i(x)}(x)}{\partial x_k}\right\|_2 \right )_{k \in \{ 1,\ldots, n \}}\in \R_{\geq 0}^n.
\]

\noindent Therefore, $\mathcal{S}_F$ assigns to each point in a persistence diagram, a real value indicating how important a given point in the diagram is for the classification.

\paragraph{Experiments}
\subparagraph{\texttt{ORBIT5k} dataset} For the classification problem of \texttt{ORBIT5k} as in \ref{s:benchmark}, we observe the traditional TDA motto that the important features in a persistence diagram are the most persistent ones. Indeed, points closer to the diagonal have almost zero saliency value; see Figure \ref{fig:sal_3_5}. 

\begin{figure}

  \includegraphics[width=0.32\textwidth]{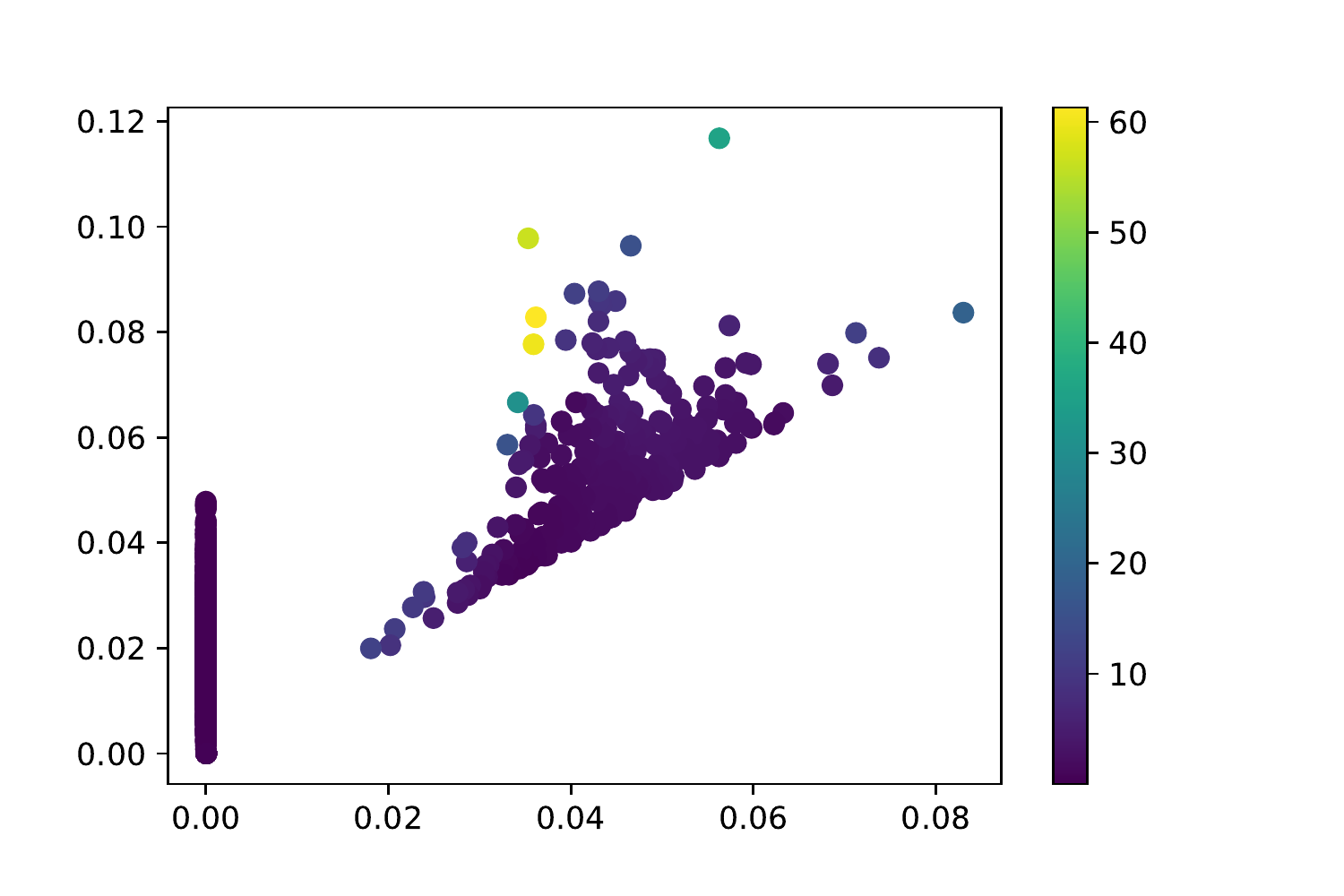}
  \includegraphics[width=0.32\textwidth]{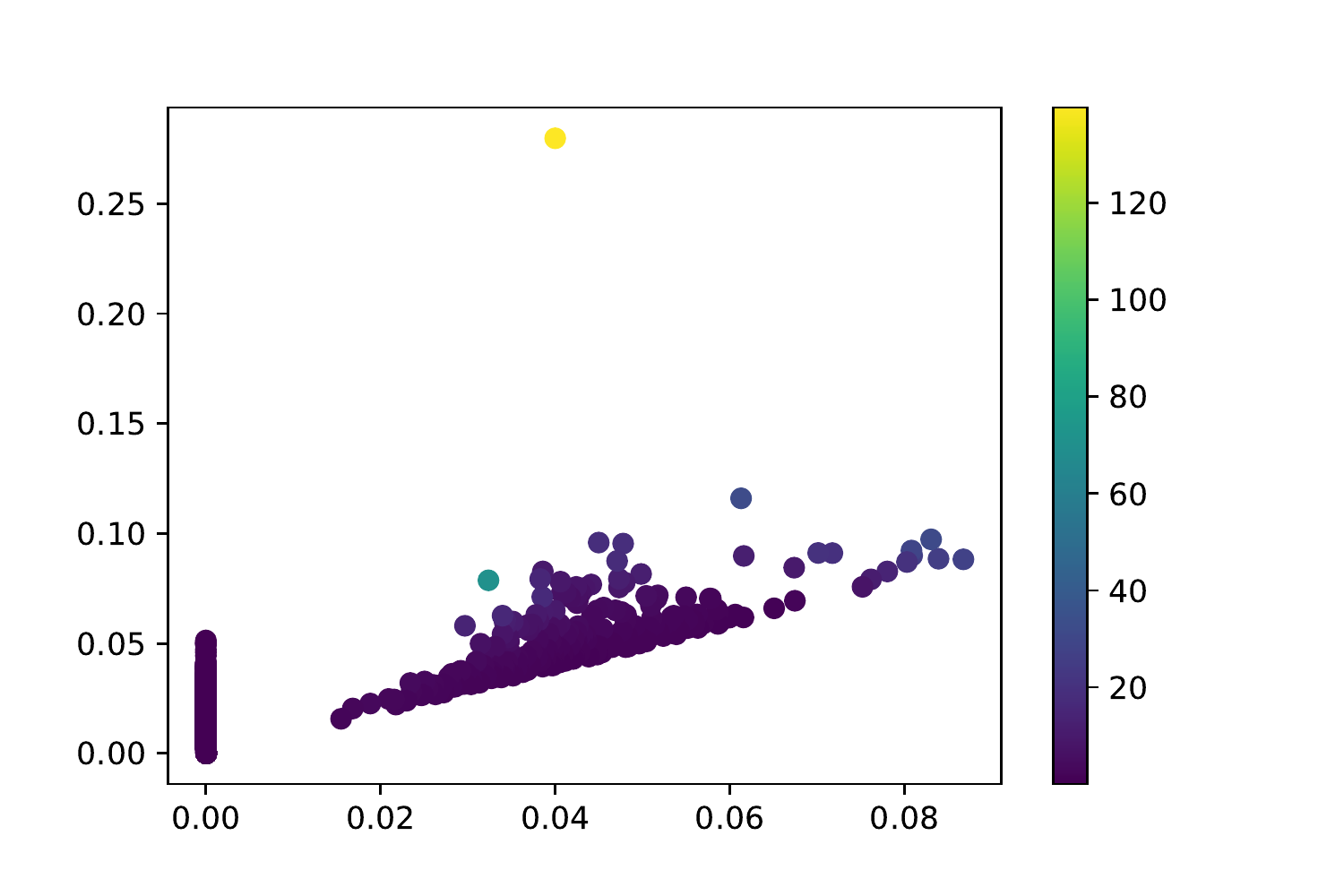}
  \includegraphics[width=0.32\textwidth]{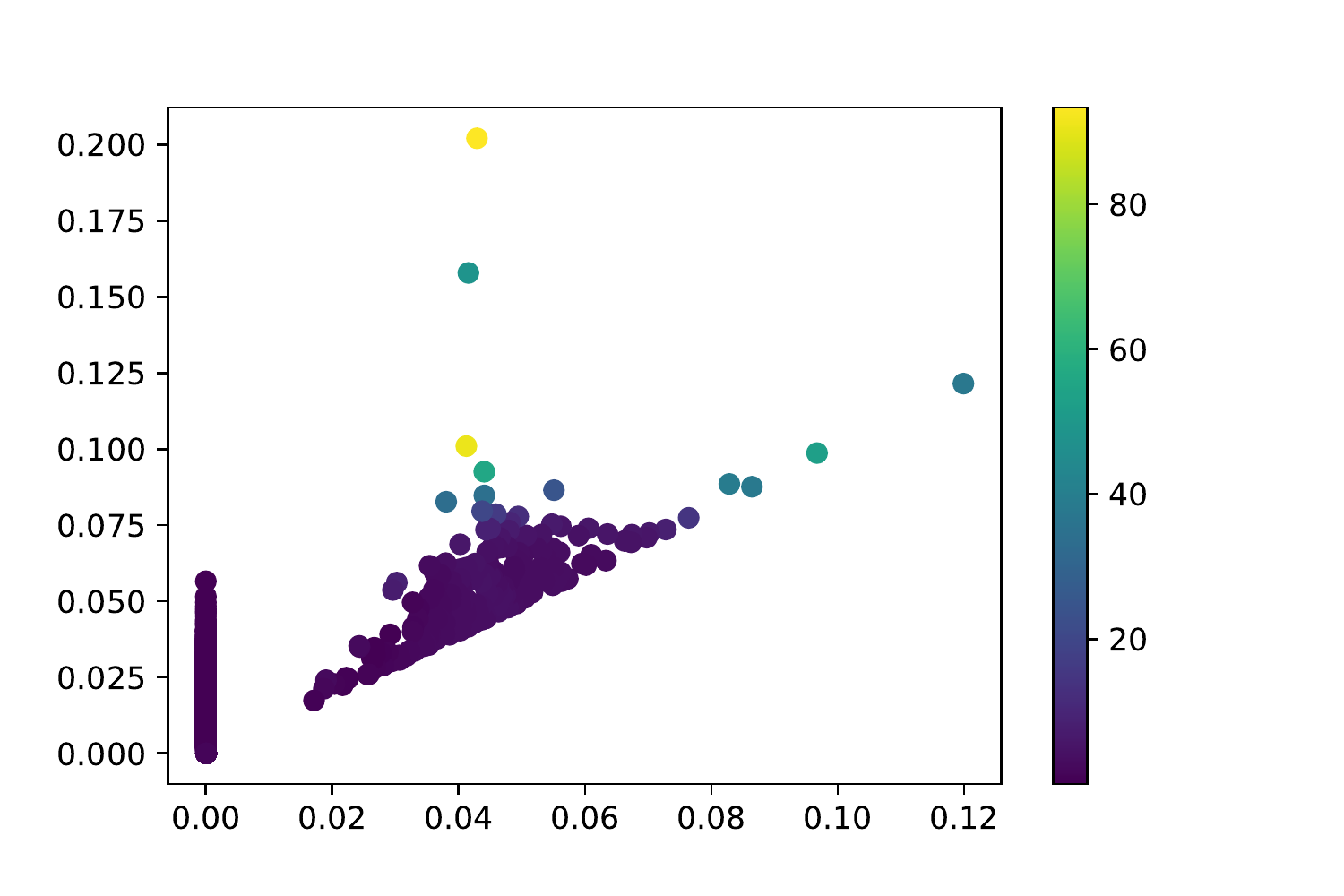}
  \caption{Saliency map of a persistence diagram of the \texttt{Orbit5k} dataset corresponding to (left to right) $\rho=3.5, 4.1, 4.3$ for both $H_0$ and $H_1$. The $H_0$ features are all points with birth value 0 and the $H_1$ features are all points with birth value $>0$. The color scale represents the Saliency map score for the classification problem described in the paper body.}\label{fig:sal_3_5}

\end{figure}

What stands out is that many stable features have a high Saliency map score, which explains the good performance of the classical vectorization methods. However, there are also points in the persistence diagram close to the diagonal with a high Saliency map score, i.e., they correspond to ``short bars'' but are very influential for the class prediction. Therefore, these are important for the classification and explain that our model outperforms previous ones. It is also remarkable that extreme points of the persistence diagram have high Saliency map scores, and $H_0$ features have a very low Saliency map score.
To empirically verify that the Saliency map scores are highly relevant for the classification, we filtered the persistence diagrams to consider only points with a Saliency map score higher than a given percentile of the saliency scores values per persistence diagram. We then evaluated the filtered dataset with our original model. When filtering all points that are above the 80th percentile obtained a test accuracy of $91.1\%$ on \texttt{ORBIT5k}, which is very close to $91.2\%$, the raw performance of our model with all points in the input persistence diagram. This indicates that all features considered in our model are relevant for the classification. Additionally, this also shows that our method allows filtering persistence diagrams. For example, only considering the points above the 80th percentile still maintains good results on the test dataset. The test accuracies on dataset \texttt{ORBIT5k} under different thresholds are displayed in Fig.~\ref{fig:featureThresholdingResults}.

\begin{figure}
    \centering
	\includegraphics[width=0.65\columnwidth]{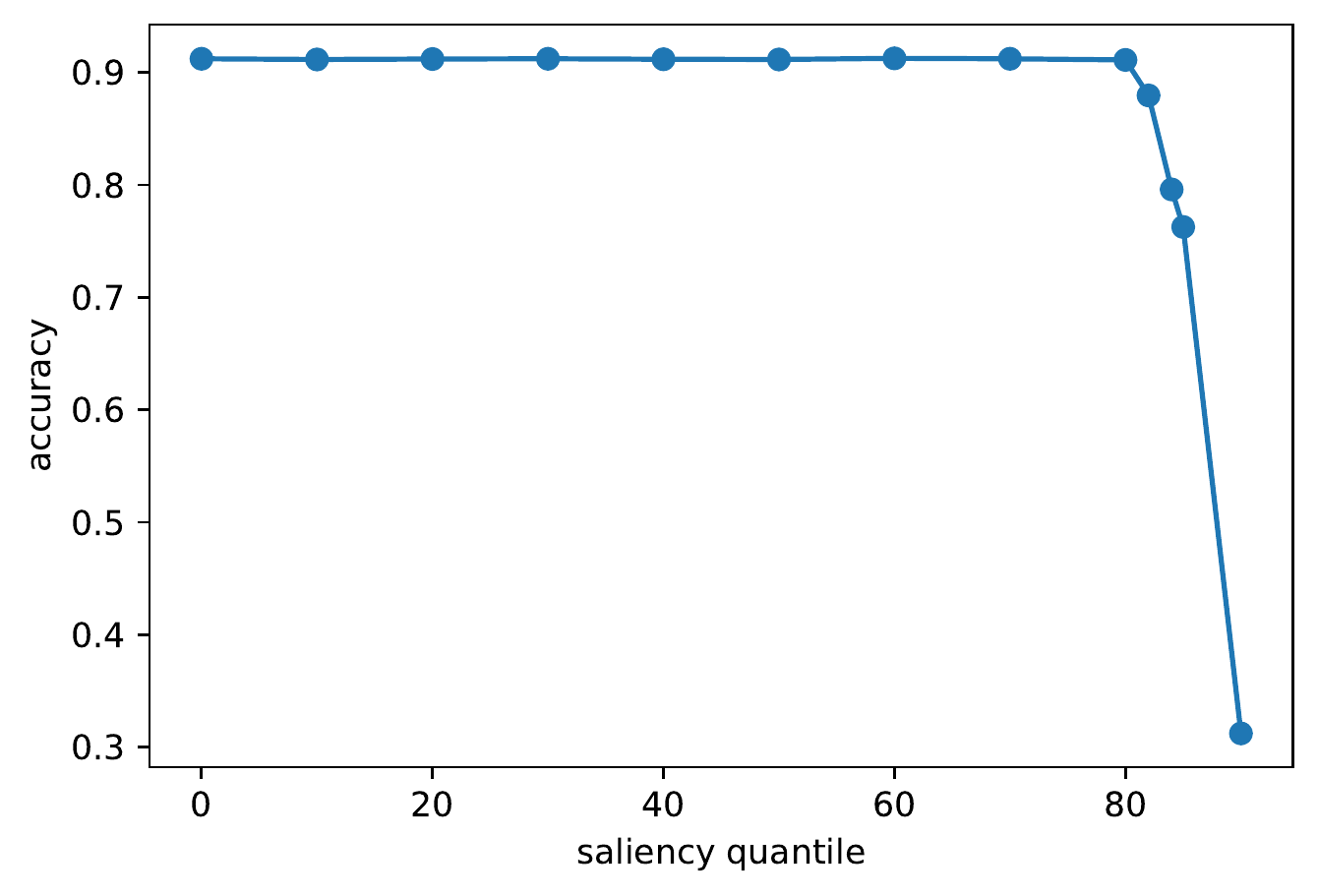}
    \caption{Test accuracies on the \texttt{ORBIT5k} dataset when considering only points above a given percentile threshold of their Saliency map score. A threshold of the 80th percentile already allows to filter the most useless points in the persistence diagram and still give very close performance to our model trained on all points.}
    \label{fig:featureThresholdingResults}
\end{figure}

We also trained a Persformer model on the filtered persistence diagram above the 80th percentile. We obtained a test accuracy of $91.2\%$, i.e., the model performs as well as the original one when filtering out points with a low Saliency map score. At the same time, we increased the computational efficiency of the Persformer model, which has a quadratic computation complexity in the persistence diagram size.

Additionally, we trained a smaller model with half the number of self-attention layers, half hidden dimension, and half the number of attention heads compared to the original model and extracted the Saliency map scores. Then we trained the original model on the filtered persistence diagram with a threshold equal to the 80th percentile of the Saliency map score of the smaller model and obtained a test accuracy of $91.0\%$. This means that we have the potential to efficiently filter out features before the training procedure using a simpler model, which still maintains the performance.

\subparagraph{Curvature dataset}
In \cite{Bubenik2020}, the authors consider the following regression problem: sample randomly 1,000 points on a disc with radius 1 of constant Gaussian curvature $K$ and predict $K$ from the Vietoris-Rips persistence diagrams of the resulting point-cloud with respect to the intrinsic metric. They give mathematical evidence that small bars in these persistence diagrams should be good predictors of the curvature, and are able to train different machine learning models with good $R^2$-score, such as Support Vector Regression. 

We reproduced the dataset of \cite{Bubenik2020} and trained a regression \texttt{Pers\-former} model with mean squared error to predict the curvature of a point cloud using the $H_1$ persistence diagrams as predictors. 
As a result, we obtain an $R^2$-score of 0.94 compared to an $R^2$-score of 0.78 using Support Vector Regression \cite{Bubenik2020}.

We calculate the Saliency map score using the Saliency map method as in the previous dataset, see Figure \ref{fig:sal_3_6}. Some points in the persistence diagram that correspond to small bars have a very high Saliency map score, which confirms the results of \cite{Bubenik2020}. An interesting observation is that the Saliency map score also seems to increase with birth time.

The importance of ``short bars" for the classifications result is much more pronounced in this case than for the \texttt{Orbit5k} dataset. Moreover, we can even show that the Saliency map score decreases with distance to the diagonal, see Figure \ref{fig:sal_dist_diag_1} and Figure \ref{fig:sal_dist_diag_2}.

To visualize the dependence of the Saliency map score on the distance to the diagonal, we normalize the lifetimes and the Saliency map scores per persistence diagram to the interval $[0,\, 1]$. Then we distribute the points in the persistence diagram to the bins $[0,\, 0.1],\ldots, [0.9,\, 1.0]$ with respect to their lifetime and consider the maximum and the sum of all Saliency map scores per bin. We then average the scores per bin over the entire test dataset. As one can clearly see, the Saliency map score decreases with the distance to the diagonal, demonstrating that ``short bars'' are important for estimating the curvature of a dataset.

\begin{figure}[!htb]

  \includegraphics[width=0.32\textwidth]{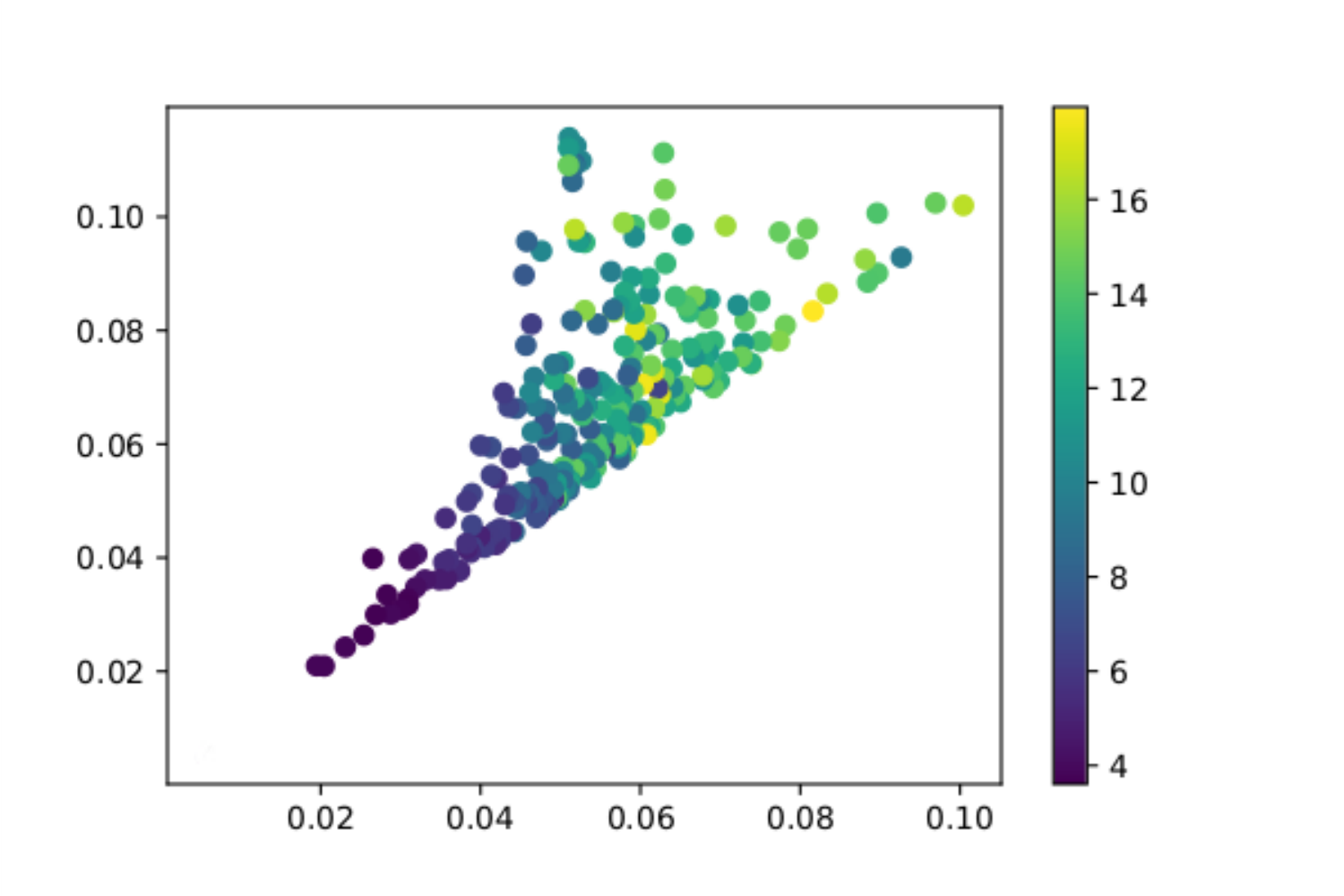}
  \includegraphics[width=0.32\textwidth]{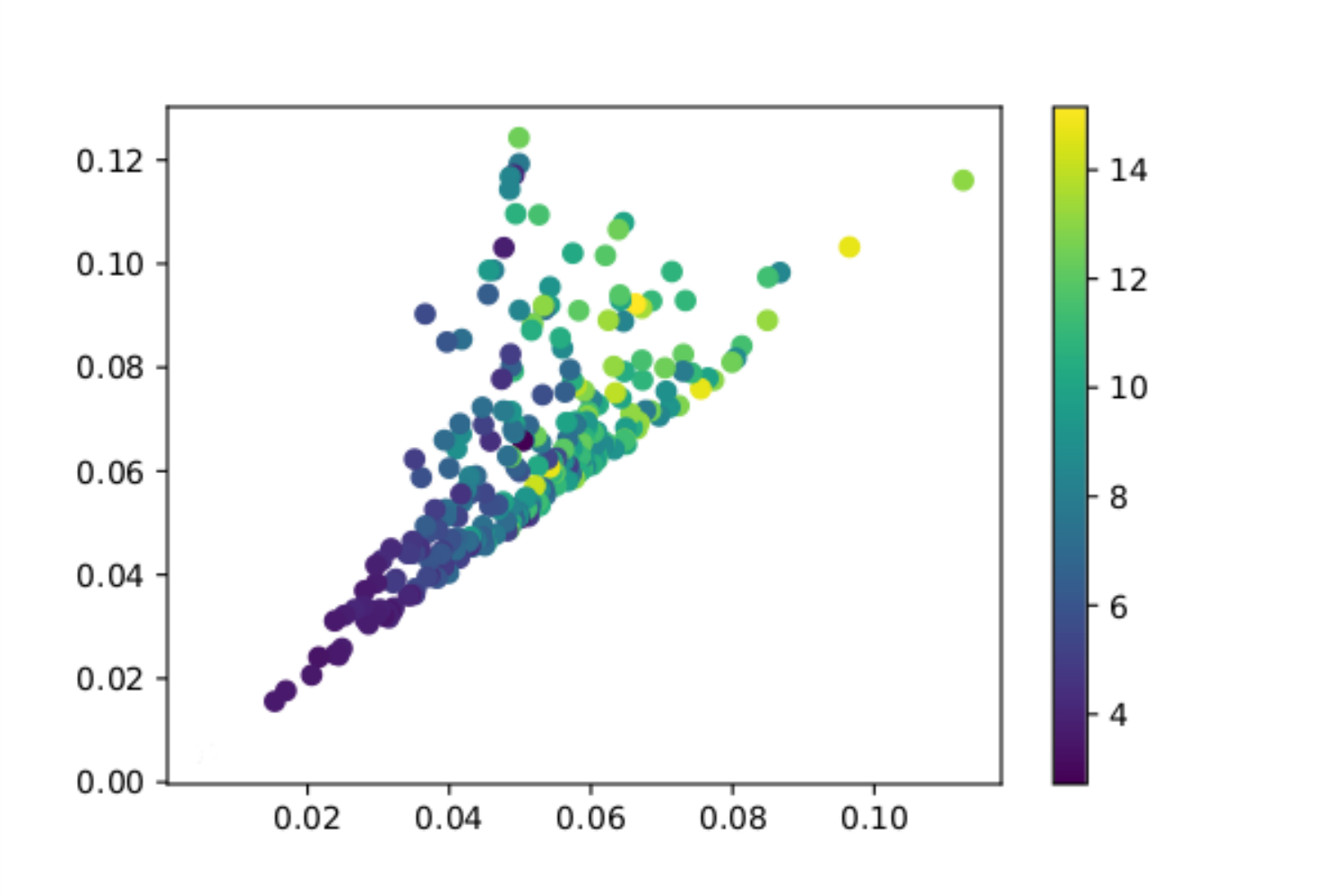}
  \includegraphics[width=0.32\textwidth]{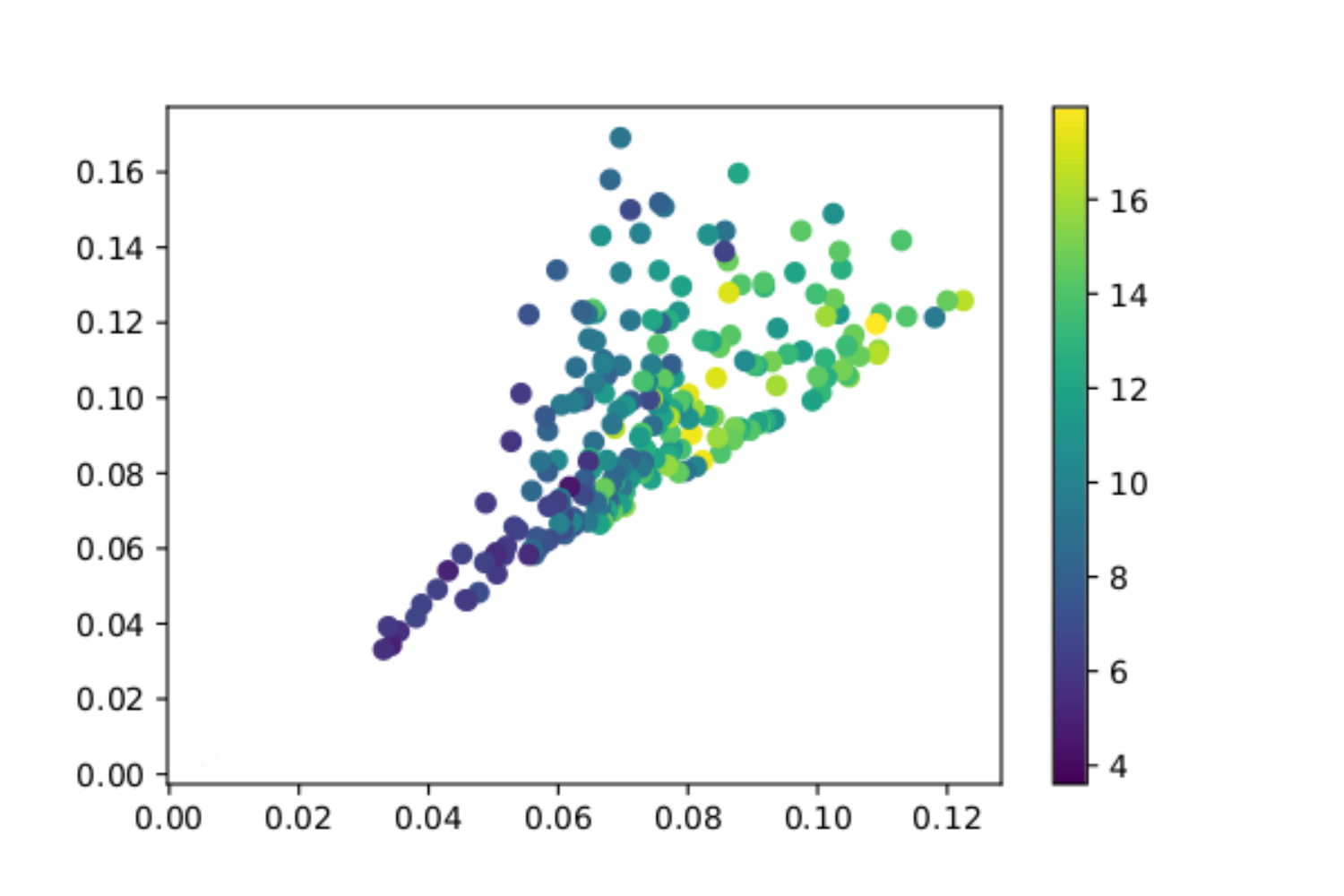}
  \caption{Saliency map of persistence diagrams of the curvature dataset of \cite{Bubenik2020} corresponding to the curvature values (left to right) $-1.50, -1.92,\  0.54$ and $H_1$.}\label{fig:sal_3_6}

\end{figure}

\begin{figure}[!htb]
\minipage{0.48\textwidth}
  \includegraphics[width=0.99\textwidth]{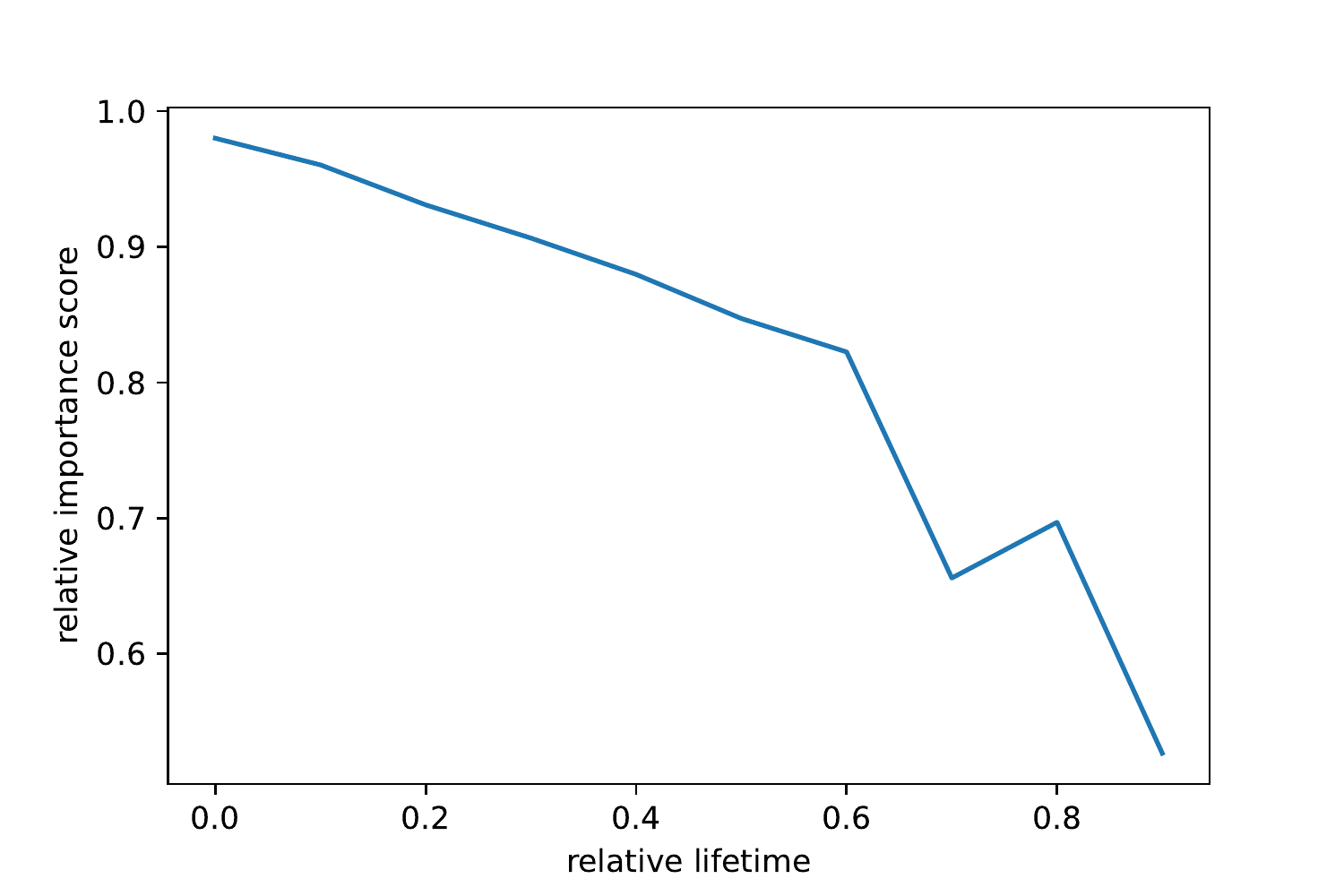}
  \caption{Relative maximum Saliency map score per bin averaged over all test persistence diagrams.}\label{fig:sal_dist_diag_1}
\endminipage\hfill
\minipage{0.48\textwidth}
  \includegraphics[width=0.99\textwidth]{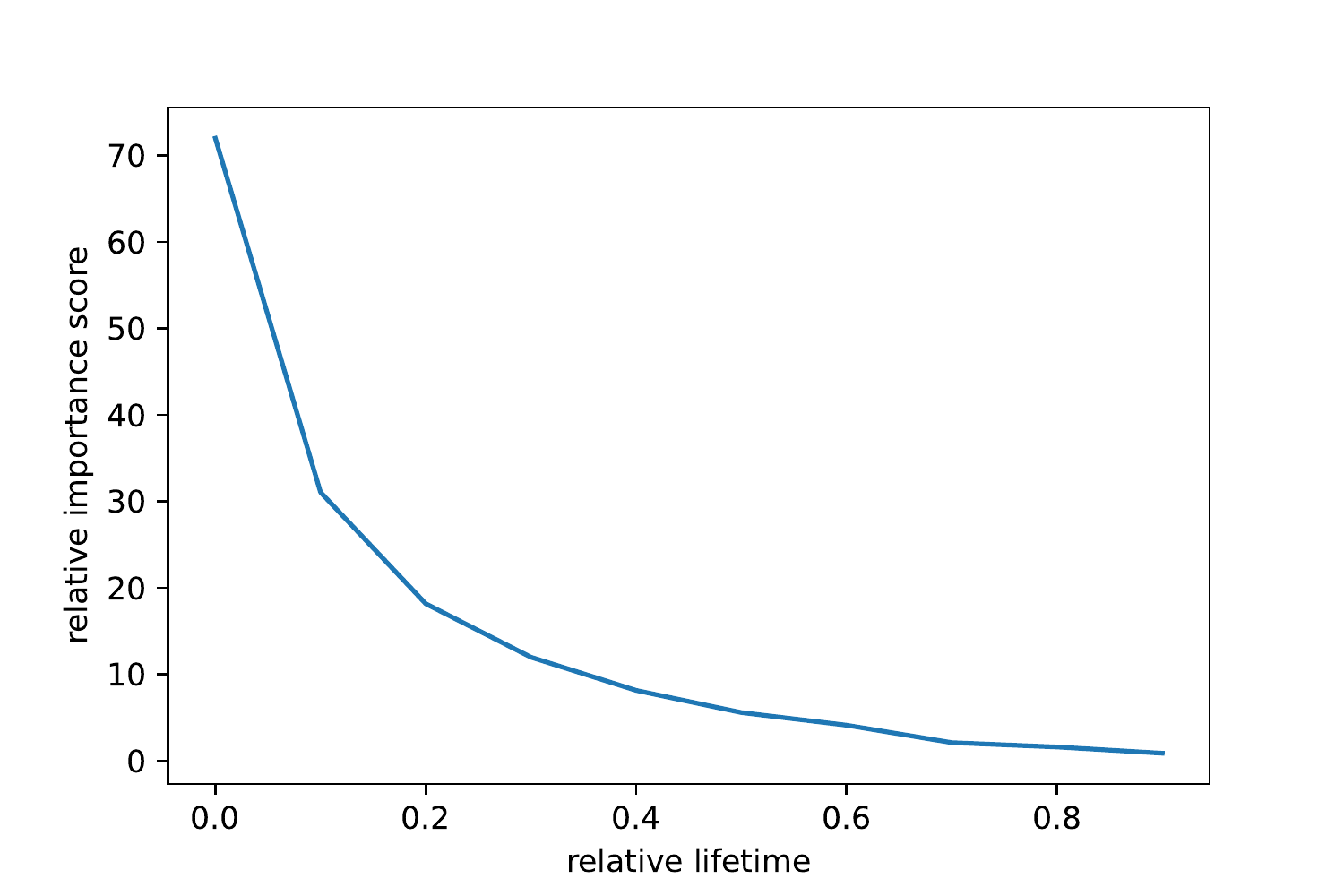}
  \caption{Sum of relative Saliency map score per bin averaged over all test persistence diagrams.}\label{fig:sal_dist_diag_2}
\endminipage
\end{figure}


\section{Conclusion}

In this work, we have introduced \texttt{Persformer}, the first Transformer architecture for learning on persistence diagrams datasets. We proved that this architecture significantly outperforms previous ones on usual benchmark synthetic and graph datasets. 

In addition,  \texttt{Persformer} is the first neural network architecture for persistence diagrams that is shown to satisfy a universal approximation theorem, enabling us to adapt Saliency Map, a well-known interpretability method, to \texttt{Persformer}. To the best of our knowledge, this is the first method for interpretable topological machine learning that allows us to highlight the topological features that matter the most to \texttt{Persformer} on specific tasks. This can be used to understand the data space and to better set up the parameters for future experiments and learning procedures. We also exhibit that Saliency Map can be used as for feature selection, reducing the data dimension significantly while not hurting the performances. We expect \texttt{Persformer} to be used in many use cases, from material science to biology, where learning on point cloud data and graphs is needed as well as understanding topological features of the data is desired.

We implemented \texttt{Persformer} within the framework of Giotto-deep\footnote{The library is available in open source at \url{https://github.com/giotto-ai/giotto-deep}.}, a toolbox for deep learning on topological data \cite{giottodeep}. Giotto-deep is under ongoing development and is the first toolbox offering seamless integration between topological data analysis and deep learning on top of PyTorch. The library aims to provide many off-the-shelf architectures to use topology both to preprocess data (with a suite of different methods available) and to use it within neural networks. It also supports benchmarking and hyperparameter optimization. The Giotto-deep library complements another toolbox, Giotto-TDA \cite{tauzin2021giottotda}, for topological data analysis which was developed by one of the authors. It implements a wide range of methods and can be used in conjunction with Giotto-deep to build end-to-end architectures, for example for classification and regression. 

\paragraph{Acknowledgements}
The authors would like to thank Kathryn Hess Bellwald for the many fruitful discussions and valuable comments.

This work was supported by the Swiss Innovation Agency (Innosuisse project 41665.1 IP-ICT).

\newpage

\bibliographystyle{alpha}

\newcommand{\etalchar}[1]{$^{#1}$}

\newpage
\begin{appendices}

\section{Numerical instability of the \texttt{Orbit5k} dataset}
\label{sec:ig}

When preparing the benchmark datasets, we realised that the non-linear dynamical system used in the \texttt{ORBIT5K} dataset exhibits chaotic behaviour. In particular, the numerical error in the generation of the \texttt{ORBIT5k} dataset increases exponentially, see Figure \ref{fig:num_stab_1} and Figure \ref{fig:num_stab_2}. Hence the topology depends heavily on the floating point precision used. We created a dataset with arbitrary precision, meaning a high enough floating point precision such that a further increase in floating point precision does not change the point cloud, and trained a model on a dataset that was created with \texttt{float64} precision.

The model generalized well to the dataset generated with arbitrary precision, with only a small difference (2.1\% accuracy points) in test performance.

\begin{figure}[!htb]
\minipage{0.48\textwidth}
  \includegraphics[width=\linewidth]{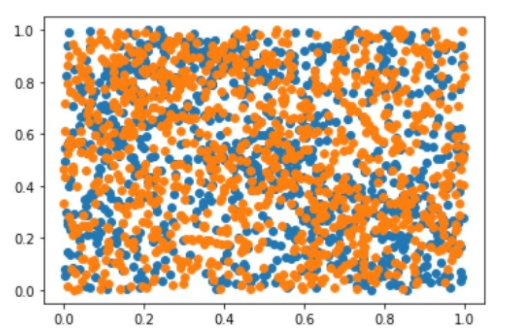}
  \caption{Orbits with the same initial point but float64 precision (orange) and infinite precision (blue)}\label{fig:num_stab_1}
\endminipage\hfill
\minipage{0.48\textwidth}
  \includegraphics[width=\linewidth]{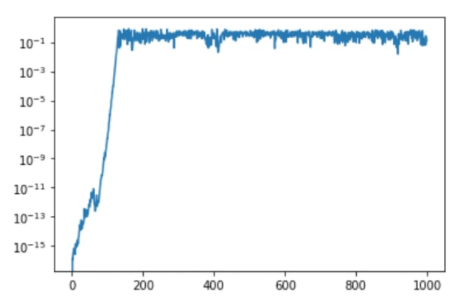}
  \caption{Distance of the $n$-th point of the orbit with arbitrary floating point precision to the $n$-th point of the orbit with float-64 precision in the quotient metric of $[0, 1]^2$ modulo $1$.}\label{fig:num_stab_2}
\endminipage
\end{figure}

\end{appendices}
\end{document}